\documentclass[]{article}

\usepackage{proceed2e}
\usepackage{amsfonts}
\usepackage{amsmath}
\usepackage{algorithmic}
\usepackage{algorithm}
\usepackage{graphicx}
\usepackage{multirow}
\usepackage[small]{caption}

\def\Rset{\mathbb{R}}

\DeclareMathOperator*{\argmin}{\rm argmin}

\DeclareMathOperator{\Tr}{Tr}

\DeclareMathOperator{\diag}{diag}

\providecommand{\frob}[2]{\langle#1, #2\rangle_F}
\providecommand{\dprod}[2]{\langle#1, #2\rangle}

\newtheorem{theorem}{Theorem}

\newtheorem{proposition}{Proposition}
\newtheorem{corollary}[theorem]{Corollary}
\newcommand{\qed}{\hfill\rule{7pt}{7pt}}
\newenvironment{proof}{\noindent{\bf Proof:}}{\qed}
\newenvironment{proof*}{\noindent{\bf Proof:}}{}

\newcommand{\E}{\mat{E}}
\newcommand{\cH}{\mathcal{H}}
\newcommand{\cG}{\mathcal{G}}

\newcommand{\cA}{\mathcal{A}}
\newcommand{\cX}{\mathcal{X}}

\newcommand{\cR}{\mathcal{R}}
\newcommand{\cW}{\mathcal{W}}
\newcommand{\mat}[1]{{\mathbf #1}}
\newcommand{\K}{\mat{K}}

\newcommand{\X}{\mat{X}}

\newcommand{\A}{\mat{A}}
\newcommand{\B}{\mat{B}}

\newcommand{\M}{\mat{M}}
\newcommand{\N}{\mat{N}}

\newcommand{\Z}{\mat{Z}}
\newcommand{\Zb}{\overline{\mat{Z}}}
\newcommand{\U}{\mat{U}}
\newcommand{\V}{\mat{V}}
\newcommand{\I}{\mat{I}}

\newcommand{\h}{\widehat}

\renewcommand{\u}{\mat{u}}
\renewcommand{\v}{\mat{v}}

\newcommand{\w}{\mat{w}}
\newcommand{\x}{\mat{x}}
\newcommand{\y}{\mat{y}}
\newcommand{\z}{\mat{z}}
\newcommand{\zb}{\overline{\mat{z}}}

\newcommand{\1}{\mat{1}}

\newcommand{\Alpha}{{\boldsymbol \alpha}}

\newcommand{\Ssigma}{{\boldsymbol \sigma}}
\newcommand{\vpsi}{{\boldsymbol \psi}}
\newcommand{\vphi}{{\boldsymbol \phi}}

\newcommand{\set}[1]{\{#1\}}
\newcommand{\hs}{\hspace{-0.28cm}}
\newcommand{\hhs}{\hspace{-0.15cm}}
\newcommand{\ignore}[1]{}

\newcommand{\Rad}{\ensuremath{\mathfrak{R}}}

\title{Learning with Missing Features}
\author{ {\bf Afshin Rostamizadeh} \\  
Dept.\ of Electrical Engineering \\
and Computer Science, \\
UC Berkeley \\ 
\And 
{\bf Alekh Agarwal}  \\ 
Dept.\ of Electrical Engineering \\
and Computer Science, \\
UC Berkeley \\ 
\And 
{\bf Peter Bartlett}   \\ 
Mathematical Sciences, QUT and \\
EECS and Statistics, \\ UC Berkeley \\ 
}

\begin{document}
\maketitle
\vspace{-2cm}
\begin{abstract}
We introduce new online and batch algorithms that are robust to data
with missing features, a situation that arises in many practical
applications.  In the online setup, we allow for the comparison hypothesis
to change as a function of the subset of features that is observed on
any given round, extending the standard setting where the comparison
hypothesis is fixed throughout. In the batch setup, we present a
convex relaxation of a non-convex problem to jointly
estimate an imputation function, used to fill in the values of missing
features, along with the classification hypothesis. We prove
regret bounds in the online setting and Rademacher complexity bounds
for the batch i.i.d.\ setting. The algorithms are tested on several
UCI datasets, showing superior performance over baseline imputation
methods.    
\end{abstract}
\section{Introduction}

Standard learning algorithms assume that each training example is
\emph{fully observed} and doesn't suffer any corruption. However, in
many real-life scenarios, training and test data often undergo some
form of corruption. We consider settings where all the features might
not be observed in every example, allowing for both adversarial and
stochastic feature deletion models. Such situations arise, for
example, in medical diagnosis---predictions are often desired using
only a partial array of medical measurements due to time or cost
constraints.  Survey data are often incomplete due to partial
non-response of participants. Vision tasks routinely need to deal with
partially corrupted or occluded images. Data collected through
multiple sensors, such as multiple cameras, is often subject to the
sudden failure of a subset of \mbox{the sensors.}

In this work, we design and analyze learning algorithms that address
these examples of learning with missing features.  The first setting
we consider is online learning where both examples and
missing features are chosen in an arbitrary, possibly adversarial,
fashion. We define a novel notion of regret suitable to the setting
and provide an algorithm which has a provably bounded regret on the
order of $O(\sqrt{T})$, where $T$ is the number of examples.
The second scenario is batch learning, where examples and
missing features are drawn according to a fixed and unknown
distribution.  We design a learning algorithm which is guaranteed to
globally optimize an intuitive objective function and which also
exhibits a generalization error on the order of $O(\sqrt{d/T})$,
where $d$ is the data dimension.  

Both algorithms are also explored empirically across several publicly
available datasets subject to various artificial and natural types of
feature corruption.  We find very encouraging results, indicating the
efficacy of the suggested algorithms and their superior performance
over baseline methods.

Learning with missing or corrupted features has a long history in statistics
\cite{little_rubin,dempster_em}, and has recieved recent attention in
machine learning~\cite{dekel_corrupted, marlin, cesa_efficient,
chechik_struct}. Imputation methods (see~\cite{little_rubin, marlin,
  dempster_em}) fill in missing values, generally independent 
of any learning algorithm, after which standard algorithms can be
applied to the data. Better performance might be expected, though, by
learning the imputation and prediction functions
simultaneously. Previous works~\cite{marlin} address this issue
using EM, but can get stuck in local optima and do not have strong 
theoretical guarantees. Our work also is different from settings where
features are missing only at test time~\cite{dekel_corrupted,
Globerson2006nightmare}, settings that give access to noisy versions
of all the features~\cite{cesa-noise} or settings where observed
features are picked by the algorithm~\cite{cesa_efficient}.

Section \ref{sec:setup} introduces both the general
online and batch settings. Sections \ref{sec:online}
and \ref{sec:batch} detail the algorithms and theoretical results
within the online and batch settings resp.  Empirical results
are presented in Section~\ref{sec:empirical}.

\section{The Setting}
\label{sec:setup}
In our setting it will be useful to denote a training instance $\x_t
\in \Rset^d$ and prediction $y_t$, as well as a corruption vector
$\z_t \in \set{0,1}^d$, where
\begin{equation*}
  [\z_t]_i = \left\{\begin{array}{cl} 0&\mbox{if feature $i$ is not
    observed,}\\ 1&\mbox{if feature $i$ is observed.}\end{array}\right.
\end{equation*}
We will discuss as specific examples both
classification problems where $y_t \in \{-1,1\}$ and regression
problems where $y_t \in \Rset$. The learning algorithm is given 
the corruption vector $\z_t$ as well as the corrupted instance,
\begin{equation*}
\x_t' = \x_t \circ \z_t \,,
\end{equation*}
where $\circ$ denotes the component-wise product between two vectors.
Note that the training algorithm is never given access to $\x_t$,
however it is given $\z_t$, and so has knowledge of exactly which
coordinates have been corrupted.  
The following subsections explain the online and batch settings
respectively, as well as the type of hypotheses that are considered
in each.

\subsection{Online learning with missing features}
\label{sec:setup-online}
In this setting, at each time-step $t$ the learning algorithm is
presented with an arbitrarily (possibly adversarially) chosen
instance $(\x_t', \z_t)$ and is expected to predict $y_t$.  After
prediction, the label is then revealed to the learner
which then can update its hypothesis.

A natural question to ask is what happens if we simply ignore
the distinction between $\x'_t$ and $\x_t$ and just run an
online learning algorithm on this corrupted data. Indeed, doing so
would give a small bound on regret:
\begin{equation}
  R(T,\ell) = \sum_{t=1}^T \ell(\dprod{\w_t}{\x_t'}, y_t) - \inf_{\w
    \in \cW} \sum_{t=1}^T \ell(\dprod{\w}{\x_t'}, y_t) \,,
  \label{eqn:stdregret}
\end{equation}
with respect to a convex loss function $\ell$ and for any convex
compact subset $\cW \subseteq \Rset^d$. However, any fixed weight
vector $\w$ in the second term might have a very large loss, making
the regret guarantee useless---both the learner and the comparator
have a large loss making the difference small. For instance, assume
one feature perfectly predicts the label, while another one only
predicts the label with 80\% accuracy, and $\ell$ is the quadratic
loss. It is easy to see that there is no fixed $\w$ that will perform
well on both examples where the first feature is observed and examples
where the first feature is missing but the second one is observed.

To address the above concerns, we consider using a linear
\emph{corruption-dependent hypothesis} which is permitted to change as
a function of the observed corruption $\z_t$. Specifically, given the
corrupted instance and corruption vector, the predictor uses a
function $\w_t(\cdot) : \{0,1\}^d \to \Rset^d$ to choose a weight
vector, and makes the
prediction $\h y_t = \dprod{\w_t(\z_t)}{\x_t'}$.  In order to provide
theoretical guarantees, we will bound the following notion of regret,
\begin{multline}
\label{eqn:onlineregret}
  \!\!\!\!\!\!R^z(T,\ell) \! = \!\!\! \sum_{t=1}^T \! \ell(\dprod{\w_t}{\x_t'},
  y_t) - \!\! \inf_{\w \in \cW} \! \sum_{t=1}^T \!
  \ell(\dprod{\w(\z_t)}{\x_t'}, y_t),
\end{multline}
where it is implicit that $\w_t$ also depends on $\z_t$ and $\cW$ now
consists of corruption-dependent hypotheses. Similar definitions of
regret have been looked at in the setting learning with side
information~\cite{CoverOr96, HazanMe2007}, but our special case admits
stronger results in terms of both upper and lower bounds.
%
In the most general case, we may consider $\cW$ as the class of all
functions which map $\set{0,1}^d \to \Rset^d$, however we show this can
lead to an intractable learning problem.  This motivates the study of
interesting subsets of this most general function class.  This is the main
focus of Section \ref{sec:online}.

\subsection{Batch learning with missing features}
\label{sec:setup-batch}

In the setup of batch learning with i.i.d.\ data, examples $(\x_t,
\z_t, y_t)$ are drawn according to a fixed but unknown distribution
and the goal is to choose a hypothesis that minimizes the expected
error, with respect to an appropriate loss function $\ell$: $\E_{\x_t,
  \z_t, y_t}[\ell(h(\x_t, \z_t), y_t)]$.

The hypotheses $h$ we consider in this scenario will be inspired by
imputation-based methods prevalent in statistics literature used to
address the problem of missing features~\cite{little_rubin}. An
imputation mapping is a function used to fill in unobserved features
using the observed features, after which the \emph{completed} examples
can be used for prediction. In particular, if we consider an
imputation function $\vphi: \Rset^d \times \set{0,1}^d \to \Rset^d$,
which is meant to fill missing feature values, and a linear predictor
$\w \in \Rset^d$, we can parameterize a hypothesis with these two
function $h_{\vphi, \w}(\x'_t, \z_t) = \dprod{\w}{\vphi(\x'_t,
  \z_t)}$.

It is clear that the multiplicative interaction between $\w$ and
$\vphi$ will make most natural formulations non-convex, and we
elaborate more on this in Section~\ref{sec:batch}. In the
i.i.d.\ setting, the natural quantity of interest is the generalization
error of our learned hypothesis. We provide a Rademacher complexity
bound on the class of $\w,\vphi$ pairs we use, thereby showing
that any hypothesis with a small empirical error will also have a
small expected loss. The specific class of hypotheses and details of
the bound are presented in Section \ref{sec:batch}.  Furthermore, the
reason as to why an imputation-based hypothesis class is not analyzed
in the more general adversarial setting will also be explained in
\mbox{that section.}

\section{Online Corruption-Based Algorithm}
\label{sec:online}
In this section, we consider the class of \emph{corruption-dependent}
hypotheses defined in Section~\ref{sec:setup-online}. Recall the
definition of regret~(\ref{eqn:onlineregret}), which we wish to
control in this framework, and of the comparator class of functions $\cW
\subseteq \set{0,1}^d \to \Rset^d$. It is clear that the function
class $\cW$ is much richer than the comparator class in the
corruption-free scenario, where the best linear predictor is fixed for
all rounds. It is natural to ask if it is even possible to prove a
non-trivial regret bound over this richer comparator class $\cW$. In
fact, the first result of our paper provides a lower bound on the
minimax regret when the comparator is allowed to pick arbitrary
mappings, i.e.\ the set $\cW$ contains all mappings. The result is
stated in terms of the minimax regret under the loss function $\ell$
under the usual (corruption-free) definition~\eqref{eqn:stdregret}:
\begin{equation*}
  R^*(T,\ell) = \inf_{\w_1 \in
    \cW}\sup_{(\x_1,\z_1,y_1)}\cdots\inf_{\w_T \in
    \cW}\sup_{(\x_T,\z_T,y_T)} R(T,\ell)
\end{equation*}

\begin{proposition}  \label{prop:lower-bound}
If $\cW = \set{0,1}^d \to \Rset^d$ the minimax value of the corruption
dependent regret for any loss function $\ell$ is lower bounded as
  \begin{multline*}
    \inf_{\w_1 \in \cW}\sup_{(\x_1,\z_1,y_1)}\cdots\inf_{\w_T \in
      \cW}\sup_{(\x_T,\z_T,y_T)}
\!\!\! R^z(T,\ell) \\\qquad\qquad\qquad\qquad= \Omega\left( 2^{d/2}
R^*\left(\frac{T}{2^{d/2}},\ell\right)\right).
  \end{multline*}
\end{proposition}

This proposition (the proof of which appears in the
appendix \cite{long_version}) shows that the minimax regret is lower
bounded by a term that is exponential in the dimensionality of the
learning problem. For most non-degenerate convex and Lipschitz losses,
$R^*(T,\ell) = \Omega(\sqrt{T})$ without further assumptions (see
e.g.~\cite{AbernethyABR2009minimax}) which yields a
$\Omega(2^{d/4}\sqrt{T})$ lower bound. The bound can be further
strengthened to $\Omega(2^{d/2}\sqrt{T})$ for linear losses which is
unimprovable since it is achieved by solving the classification
problem corresponding to each pattern independently. 

Thus, it will be difficult to
achieve a low regret against arbitrary maps from
$\{0,1\}^d$ to $\Rset^d$. In the following section we consider a
restricted function class and show that a mirror-descent
algorithm can achieve regret polynomial in $d$ and sub-linear in $T$,
implying that the average regret is vanishing.

\subsection{\mbox{Linear Corruption-Dependent Hypotheses}}
Here we analyze a corruption-dependent hypothesis class that is
parametrized by a matrix $\A \in \Rset^{d \times k}$, where $k$ may be
a function of $d$. In the simplest case of $k = d$, the
parametrization looks for weights $\w(\z_t)$ that depend linearly on
the corruption vector $\z_t$. Defining $\w_{\A}(\z_t) = \A\z_t$
achieves this, and intuitively this allows us to capture how the
presence or absence of one feature affects the weight of another
feature. This will be clarified further in the examples.

In general, the matrix $\A$ will be $d\times k$, where $k$ will be
determined by a function $\vpsi(\z_t) \in \set{0,1}^k$ that maps
$\z_t$ to a possibly higher dimension space.  Given, a fixed $\vpsi$,
the explicit parameterization in terms of $\A$ is,
\begin{equation}
  \w_{\A, \vpsi}(\z_t) = \A \vpsi(\z_t) \,.
\end{equation}
In what follows, we drop the subscript from $ \w_{\A, \vpsi}$ in order
to simplify notation. Essentially this allows us to introduce
non-linearities as a function of the corruption vector, but the
non-linear transform is known and fixed throughout the learning
process. Before analyzing this setting, we give a few examples and
intuition as to why such a parametrization is useful. In each example,
we will show how there exists a choice of a matrix $\A$ that captures
the specific problem's assumptions. This implies that the fixed
comparator can use this choice in hindsight, and by having a low
regret, our algorithm would implicitly learn a hypothesis close to
this reasonable choice of $\A$.

\subsubsection{Corruption-free special case}
We start by noting that in the case of no corruption (i.e.  $\forall
t, \z_t = \1$) a standard linear hypothesis model can be cast within
the matrix based framework by defining $\vpsi(\z_t) = 1$ and learning
$\A \in \Rset^{d \times 1}$.

\subsubsection{Ranking-based parameterization}
One natural method for classification is to order the features by their
predictive power, and to weight features proportionally to their
ranking (in terms of absolute value; that is, the sign of weight
depends on whether the correlation with the label is positive or
negative). In the corrupted features setting, this naturally
corresponds to taking the available features at any round and putting
more weight on the most predictive observed features. This is
particularly important while using margin-based losses such as the
hinge loss, where we want the prediction to have the right sign and be
large enough in magnitude. 

Our parametrization allows such a strategy when using a simple
function $\vpsi(\z_t) = \z_t$. Without loss of generality, assume that
the features are arranged in decreasing order of discriminative power
(we can always rearrange rows and columns of $\A$ if they're not). We
also assume positive correlations of all features with the label; a more
elaborate construction works for $\A$ when they're not. In this case,
consider the parameter matrix and the induced classification weights
\begin{align*}
  [\A]_{i,j} = \left\{  \hhs
  \begin{array}{rl}
      1, & \hhs j = i \\
     -\frac{1}{d}, & \hhs j < i \\
      0, & \hhs  j > i
  \end{array}
  \right.\!\!, ~~[\w(\z_t)]_i \! = \! [\z_t]_i\biggr(1 - \!\!\!\!\!
  \sum_{\substack{j < i :\\ [\z_t]_j = 1}} \frac{1}{d}\biggr). 
\end{align*}
Thus, for all $i < j$ such that $[\z_t]_i = [\z_t]_j = 1$ we have
$[\w(\z_t)]_i \geq [\w(\z_t)]_j$. The choice of 1 for diagonals and
$1/d$ for off-diagonals is arbitrary and other values might also be
picked based on the data sequence $(\x_t,\z_t,y_t)$. In general,
features are weighted monotonically with respect to their
discriminative power with signs based on correlations with the label. 

\subsubsection{Feature group based parameterization}
Another class of hypotheses that we can define within this framework
are those restricted to consider up to $p$-wise interactions between
features for some constant $0 < p \leq d$. In this case, we index the
$k = \sum_{i=1}^p \binom{d}{i} = O\big((\frac{d}{p})^p\big)$ unique
subsets of features of size up to $p$. Then define $[\vpsi(\z_t)]_j =
1$ if the corresponding subset $j$ is uncorrupted by $\z_t$ and equal
to $0$ otherwise. An entry $[\A]_{i,j}$ now specifies the importance of
feature $j$, assuming that at least the subset $i$ is present.  Such a
model would, for example, have the ability to capture the scenario of
a feature that is only discriminative in the presence of some $p-1$
other features. For example, we can generalize the ranking example
from above to impose a soft ranking on groups of features.


\subsubsection{Corruption due to failed sensors}
\label{sec:failed_sensors}

A common scenario for missing features arises in applications
involving an array of measurements, for example, from a sensor network,
wireless motes, array of cameras or CCDs, where each sensor is bound to fail
occasionally. The typical strategy for dealing with such situations
involves the use of redundancy. For instance, if a sensor fails, then
some kind of an averaged measurement from the neighboring sensors
might provide a reasonable surrogate for the missing value. 

It is possible to design a choice of $\A$ matrix for the comparator
that only uses the local measurement when it is present, but uses an
averaged approximation based on some fixed averaging distribution on
neighboring features when the local measurement is missing. For each
feature, we consider a probability distribution $p_i$ which specifies the
averaging weights to be used when approximating feature $i$ using
neighboring observations. Let $\w^*$ be the weight vector that the
comparator would like to use if all the features were present. Then,
with $\vpsi(\z) =\z$ and for $j \neq i$ we define,
\begin{equation} \label{eqn:Amatrixlocal}
  [\A]_{i,i} = \w^*_i + \sum_{j\ne i}\w^*_jp_{ji},\quad [\A]_{i,j} =
  -\w^*_jp_{ji}.
\end{equation}
Thus, say only feature $k$ is missing, we still
have ${\x'}_t^\top \A \z_t = \sum_{i,j} [\x'_t]_i [\z_t]_j [\A]_{i,j} =
\sum_{i\neq k,j\neq k} [\x_t]_i [\A]_{i,j} = \sum_{i\neq k} [\x_t]_i
[\w^*]_i + [\w^*]_k \sum_{i \neq k} [\x_t]_i p_{ki}$, where by assumption
$\sum_{i \neq k} [\x_t]_i p_{ki} \approx [\x_t]_k$.

Of course, the averaging in such applications is typically local, and
we expect each sensor to put large weights only on neighboring
sensors. This can be specified via a neighborhood graph, where nodes
$i$ and $j$ have an edge if $j$ is used to predict $i$ when feature
$i$ is not observed and vice versa. From the
construction~\eqref{eqn:Amatrixlocal} it is clear that the only
off-diagonal entries that are non-zero would correspond to the
edges in the neighborhood graph. Thus we can even add this
information to our algorithm and constrain several off-diagonal
elements to be zero, thereby restricting the complexity of the
problem.

\subsection{Matrix-Based Algorithm and Regret}
\label{sec:matrix-alg}
We use a standard mirror-descent style 
algorithm~\cite{yudin83book, beck2003mirror} in the matrix based
parametrization described above. It is characterized by a
strongly convex regularizer $\cR~:~\Rset^{d \times k}\to\Rset$, that is
\vspace{-0.2cm}
\begin{small}
\begin{equation*}
  \cR(\A) \geq \cR(\B) + \dprod{\nabla\cR(\B)}{\A - \B}_F + \frac{1}{2}\|\A
  - \B\|^2~~\forall\A, \B \! \in \! \cA,
  \vspace{-0.2cm}
\end{equation*}
\end{small}
for some norm $\|\cdot\|$ and where $\dprod{\A}{\B}_F = \mathrm{Tr}(\A^\top\B)$ is the trace
inner product. An example is the squared Frobenius norm
$\cR(\A) = \frac{1}{2}\|\A\|_F^2$. For any such function, we can
define the associated Bregman divergence 
\begin{align*}
  D_{\cR}(\A,\B) = \cR(\A) - \cR(\B) - \dprod{\nabla\cR(\B)}{\A -
\B}_F .
\end{align*}
We assume $\cA$ is a 
convex subset of $\Rset^{d\times k}$, which could encode
constraints such as some off-diagonal entries being zero in the setup
of Section~\ref{sec:failed_sensors}. To simplify presentation in what
follows, we will use the shorthand $\ell_t(\A) =  \ell(\dprod{\A
  \vpsi(\z_t)}{\x'_t}, y_t)$. The algorithm initializes with any $\A_0 \in
\cA$ and updates
\begin{equation}
  \A_{t+1} \!\!=\! \arg\min_{\A \in
    \cA}\left\{\eta_t\dprod{\nabla\ell_t(\A_t)}{\A}_F
   \!+\! D_{\cR}(\A, \A_t)\right\}
   \label{mirror_descent}
\end{equation}
If $\cA = \Rset^{d\times k}$ and $\cR(\A) =
\frac{1}{2}\|\A\|_F^2$, the update simplifies to gradient descent
$\A_{t+1} = \A_t - \eta_t\nabla\ell_t(\A_t)$. 

Our main result of this section is a guarantee on the regret incurred
by Algorithm~\eqref{mirror_descent}. The proof follows from standard
arguments (see e.g.~\cite{yudin83book,CesaBianchiLugosi06book}). Below, the dual norm is defined
as $\|\V\|_* = \sup_{\U : \|\U\| \leq 1} \dprod{\U}{\V}_F$.
%
%
\begin{theorem}
  Let $\cR$ be strongly convex with respect to a norm
  $\|\cdot\|$ and $\|\nabla \ell_t(\A)\|_* \leq G$, then
  Algorithm \ref{mirror_descent} with learning rate 
  $\eta_t = \frac{R}{G \sqrt{T}}$ exhibits the following regret
  upper bound compared to any $\A$ with $\|\A\|\leq R$, 
  \begin{equation*}
\sum_{t=1}^T \! \ell(\dprod{\A_t \z_t}{\x_t'}, y_t) -
  \!\! \inf_{\A \in \cA} \sum_{t=1}^T \! \ell(\dprod{\A \z_t}{\x_t'}, y_t)
    \leq 3RG\sqrt{T}.
  \end{equation*}
  \label{thm:mmdbound}
\end{theorem}
%

\section{Batch Imputation Based Algorithm}
\label{sec:batch}
Recalling the setup of Section~\ref{sec:setup-batch}, in this section
we look at imputation mappings of the form  
\begin{equation}
  \vphi_{\M}(\x', \z) = \x' + \diag(1-\z)\M^\top\x'\,.
  \label{eqn:linearimp}
\end{equation}
Thus we retain all the observed entries in the vector $\x'$, but for
the missing features that are predicted using a linear combination
of the observed features and where the $i_{th}$ column of $\M$ encodes the
averaging weights for the $i_{th}$ feature. Such a linear prediction
framework for features is natural. For instance, when the data vectors
$\x$ are Gaussian, the conditional expectation of any feature given
the other features is a linear function. The predictions are now
made using the dot product
\begin{equation*}
  \dprod{\w}{\vphi(\x',\z)} = \dprod{\w}{\x'} +
  \dprod{\w}{\diag(1-\z)\M^\top\x'},
\end{equation*}
where we would like to estimate $\w, \M$ based on the data
samples. From a quick inspection of the resulting learning problem, 
it becomes clear that optimizing over such a hypothesis class
leads to a non-convex problem. The convexity of the loss plays a
critical role in the regret framework of online learning, which is why
we restrict ourselves to a batch i.i.d.\ setting here. 

In the sequel we will provide a convex relaxation to the learning
problem resulting from the
parametrization~\eqref{eqn:linearimp}. While we can make this
relaxation for natural loss functions in both classification and
regression scenarios, we restrict ourselves to a linear regression
setting here as the presentation for that example is simpler due to the
existence of a closed form solution for the ridge regression
problem. 

In what follows, we consider only the corrupted data and thus simply
denote corrupted examples as $\x_i$.  Let $\X$ denote the matrix with
$i_{th}$ row equal to $\x_i$ and similarly define $\Z$ as the matrix
with $i_{th}$ row equal to $\z_i$.  It will also be useful to define
$\Zb = \1\1^\top - \Z$ and $\zb_i = \1 - \z_i$ and finally let $\Zb_i
= \diag(\zb_i)$.

\subsection{Imputed Ridge Regression (IRR)}
\label{sec:imputation-alg}
In this section we will consider a modified version of the
ridge regression (RR) algorithm, robust to missing features. The
overall optimization problem we are interested in is as follows,
\begin{small}
\begin{align}
\hs \min_{\{\w,\M:\|\M\|_F \leq \gamma\}} \!
\frac{\lambda}{2} \| \w \|^2 \!+\! \frac{1}{T} \sum_{i=1}^T \! \big(y_i \!-\!
\w^\top \!(\x_i \!+\! \Zb_i \M^\top\x_i)\big)^2
\label{irr_primal}
\end{align}
\end{small}
where the hypothesis $\w$ and imputation matrix $\M$ are
simultaneously optimized.  In order to bound the size of the
hypothesis set, we have introduced the constraint $\|\M\|_F^2 \leq
\gamma^2$ that bounds the Frobenius norm of the imputation matrix.
The global optimum of the problem as presented in (\ref{irr_primal})
cannot be easily found as it is not jointly convex in both $\w$ and
$\M$. We next present a convex relaxation of the
formulation~\eqref{irr_primal}. The key idea is to take a dual over
$\w$ but not $\M$, so that we have a saddle-point problem in the dual
vector $\Alpha$ and $\M$. The resulting saddle point problem, while
being concave in $\Alpha$ is still not convex in $\M$. At this step we
introduce a new tensor $\N \in \Rset^{d\times d\times d}$, where
$[\N]_{i,j,k} = [\M]_{i,k}[\M]_{j,k}$. Finally we drop the non-convex
constraint relating $\M$ and $\N$ replacing it with a matrix positive
semidefiniteness constraint. 

Before we can describe the convex relaxation, we need one more piece
of notation. Given a matrix $\M$ and a tensor $\N$, we define the
matrix $\K_{\M\N} \in \Rset^{T\times T}$
\begin{small}
\begin{multline}
  [\K_{\M\N}]_{i,j} = \x_i^\top \x_j
  + \x_i^\top \M \Zb_i \x_j
  + \x_i^\top \Zb_j \M^\top \x_j \\
  + \sum_{k=1}^d [\zb_i]_k [\zb_j]_k \x_i^\top \N_k
    \x_j \,.
\label{eqn:relaxedkernel}
\end{multline}
\end{small}
The following proposition gives the convex relaxation of the
problem~\eqref{irr_primal} that we refer to as Imputed Ridge Regression
(IRR) and which includes a strictly larger hypothesis than the $(\w,
\M)$ pairs with which we began.

\begin{proposition}
\label{prop:irr_relaxed}
The following semi-definite programming optimization problem provides
a convex relaxation to the non-convex problem (\ref{irr_primal}):
\begin{align}
  \label{irr_relaxed}
 & \min_{\substack{t,~ \M:\|\M\|^2_F\leq\gamma^2 \\ \N: \sum_k \|\N_k\|_F^2 \leq \gamma^4}} 
   t \\
  & \mathrm{s.t.} 
     ~~ \left[
    \begin{array}{cc}
      \K_{\M\N} + \lambda T \I & \y \\
      \y^\top & t
    \end{array}
    \right] \succeq 0, ~~
    \K_{\M\N} \succeq 0 \nonumber \,.
\end{align}
\label{prop:irr_relaxation}
\end{proposition}
The proof is deferred to the appendix for lack of space. The main idea
 is to take the quadratic form that arises in the dual
formulation of~\eqref{irr_primal} with the matrix $\K_\M$,
\begin{small}
\begin{multline*}
\!\!\!\!\!\! [\K_{\M}]_{i,j}\! = \!\x_i^\top \x_j
\!+\! \x_i^\top \M \Zb_i \x_j
\!+\! \x_i^\top \Zb_j \M^\top \x_j 
\!+\! \x_i^\top \M \Zb_i \Zb_j \M^\top \x_j\!,
\end{multline*}
\end{small}
\noindent and relax it to the matrix $\K_{\M\N}$~\eqref{eqn:relaxedkernel}. The
constraint involving positive semidefiniteness of $\K_{\M\N}$ is
needed to ensure the convexity of the relaxed problem. The norm constraint
on $\N$ is a consequence of the norm constraint on $\M$. 

One tricky issue with relaxations is using the relaxed solution 
in order to find a good solution to the original
problem. In our case, this would correspond to finding a good $\w, \M$
pair for the primal problem~\eqref{irr_primal}. We bypass this step,
and instead directly define the prediction on any point $(\x_0,\z_0)$ as:
\begin{multline}
\sum_{i=1}^T \alpha_i( \x_i^\top \x_0+ \x_i^\top \M \Zb_i \x_0 
+ \x_i^\top \Zb_0 \M^\top \x_0  \\
+ \sum_{k=1}^d [\zb_i]_k [\zb_0]_k \x_i^\top \N_k \x_0).
\label{eqn:dualpredictor}
\end{multline}
Here, $\Alpha, \M, \N$ are solutions to the saddle-point problem 
\begin{align}
\label{eqn:irr_saddlepoint}
 \min_{\substack{\M:\|\M\|_F\leq\gamma \\ \N: \sum_k \|\N_k\|_F^2 \leq \gamma^4}} 
  \!\!\!\!\! \max_\Alpha 
    2 \Alpha^\top \y  \! - 
  \! \Alpha^\top (\K_{\M\N} \!+\! \lambda T \I) \Alpha \,.
\end{align}
We start by noting that the above optimization problem is equivalent
to the one in Proposition~\ref{prop:irr_relaxation}. The intuition
behind this definition~\eqref{eqn:dualpredictor} is that the solution to the
problem~\eqref{irr_primal} has this form, with $[\N]_{i,j,k}$ replaced
with $[\M]_{i,k}[\M]_{j,k}$. In the next section, we show a Rademacher
complexity bound over functions of the form above to justify our
convex relaxation.

\subsection{Theoretical analysis of IRR}

As mentioned in the previous section, we predict with a hypothesis of
the form~\eqref{eqn:dualpredictor} rather than going back to the primal
class indexed by $(\w, \M)$ pairs. In this section, we would like to
show that the new hypothesis class parametrized by $\Alpha, \M, \N$ is
not too rich for the purposes of learning. To do this, we give the
class of all possible hypotheses that can be the solutions to the dual
problem~\eqref{irr_relaxed} and then prove a Rademacher complexity
bound over that class. 
The set of all possible $\Alpha, \M, \N$ triples that can be potential
solutions to~\eqref{irr_relaxed} lie in the following set
\begin{small}
\begin{multline*}
\!\!\! \cH \! = \!  \Bigg\{\! h(\x_0, \z_0) \! \mapsto \!\! \sum_{i=1}^T \!\! \alpha_i( 
  \x_i^\top \x_0
  +  
  \x_i^\top \M \Zb_i \x_0 
  + \x_i^\top \Zb_0 \M^\top \x_0 
  +\\ \sum_{k=1}^d [\zb_i]_k [\zb_0]_k \x_i^\top \N_k \x_0)
  \!:\! \|\M\|_F \!\leq\! \gamma, \|\N\|_F \!\leq\! \gamma^2, 
    \|\Alpha\| \!\leq\! \frac{B}{\lambda \sqrt{T}}
 \! \Bigg\}
\end{multline*}
\end{small}
The bound on $\|\Alpha\|$ is made implicitly in the optimization
problem (assuming the training labels are bounded $\forall i, |y_i|
\leq B$). To see this, we note that the problem~\eqref{irr_relaxed}
is obtained from~\eqref{eqn:irr_saddlepoint} by using the closed-form
solution of the optimal $\Alpha = (\K_{\M\N} + \lambda T \I)^{-1}
\y$. Then we can bound $\|\Alpha\| \leq \|\y\| /
\lambda_{\min}(\K_{\M\N}
+ \lambda T \I) = \frac{B \sqrt{T}}{ \lambda T}$, where
$\lambda_{\min}(\A)$ denotes the smallest eigenvalue of the matrix
$\A$.  Note that in general there is no linear hypothesis $\w$ that
corresponds to the hypotheses in the relaxed class $\cH$ and that we
are dealing with a strictly more general function class.  However, the
following theorem demonstrates that the Rademacher complexity of this
function class is reasonably bounded in terms of the number of
training points $T$ and dimension $d$ and thereby still provides
provable generalization performance \cite{bm_rademacher}.

Recall the Rademacher complexity of a class $\cH$
\begin{equation}
  \Rad_T(\cH) = \E_S\E_{\Ssigma} \left[ \frac{1}{T} \sup_{h \in \cH} \bigg|
    \sum_{i=1}^T \sigma_i h(\x_i,\z_i) \bigg| \right] \,,
\end{equation} where the inner expectation is over independent Rademacher
random variables $(\sigma_1,\ldots,\sigma_T)$ and the outer one over
a sample $S = ((\x_1,\z_1),\ldots,(\x_T,\z_T))$.  
\begin{theorem}
\label{thm:rademacher}
If we assume a bounded regression problem $\forall y, ~|y| \leq B$
and $\forall \x,~ \|\x\| \leq R$, then the Rademacher complexity of the
hypothesis set $\cH$ is bounded as follows,
\begin{equation*}
  \Rad_T(\cH) \leq \big(1 + \gamma + (\gamma + \gamma^2) \sqrt{d} \big)
\frac{BR^2 }{\lambda \sqrt{T}} = O\bigg(\sqrt{\frac{d}{T}}\bigg) \,.
\end{equation*}
\end{theorem}
Due to space constraints, the proof is presented in the appendix.
%
Theorem~\ref{thm:rademacher} allows us to control the gap between
empirical and expected risks using standard Rademacher
complexity results. Theorem 8 of~\cite{bm_rademacher}, immediately
provides the following corollary.

\begin{corollary} \label{corr:unif_dev}
  Under the conditions of Theorem 2, for any $0 < \delta \leq 1$, with
  probability at least $1-\delta$ over samples of size $T$, every $h
  \in \cH$ satisfies 
  \begin{align*}
    \E&[(y - h(\x',\z))^2] \leq \frac{1}{T}\sum_{t=1}^T(y_t -
    h(\x'_t,\z_t))^2 \\&+ \frac{BR^2(1+\gamma)^2}{\lambda}
      \left(\frac{BR^2(1+\gamma)^2 }{\lambda} \sqrt{\frac{d}{T}}
  + \sqrt{\frac{8\ln(2/\delta)}{T}}\right).
  \end{align*}
\end{corollary}

\ignore{
\subsection{TODO}
\begin{itemize}
  \item Can we parametrize $M$ matrix in terms of a covariance matrix
(and corruption vector)?
  \item Can we show an online analysis?
  \item Can we show a lower bound on the Rademacher complexity of the
non-relaxed hypothesis class $\cH$ that is close to the upper bound of
the relaxed class $\cG$?
\end{itemize}
}

\section{Empirical Results}
\label{sec:empirical}

This section presents empirical evaluation of the online matrix-based
algorithm~\ref{mirror_descent}, as well as the Imputed Ridge
Regression algorithm of Section~\ref{sec:imputation-alg}.
We use baseline methods \emph{zero-imputation} and \emph{mean-imputation}
where the missing entries are replaced with zeros and mean estimated
from observed values of those features resp. Once the data is
imputed, a standard online gradient descent algorithm or
ridge-regression algorithm is used.  As reference, we also show
the performance of a standard algorithm on uncorrupted data. The
algorithms are evaluated on several UCI repository datasets,
summarized in Table~\ref{table:data}. 

The {\tt thyroid} dataset includes naturally
corrupted/missing data.  The {\tt optdigits} dataset is subjected to
artificial corruption by deleting a column of pixels, chosen uniformly
at random from the 3 central columns of the image (each image contains
8 columns of pixels total).  The remainder of the datasets are subjected to two
types of artificial corruption: \emph{data-independent} or
\emph{data-dependent} corruption. In the first case, each feature is
randomly deleted independently, while the features are deleted based
on thresholding values in the latter case. 

\begin{table}
\begin{center}
{\small
\begin{tabular}{l|llcc}
dataset & $m$ & $d$ & $F_I$ & $F_D$  \\
\hline
\hline
{\tt abalone} & 4177 & 7 & $.62 \pm .08$ & $.61 \pm .12$ \\
{\tt housing} & 20640 & 8 & $.64 \pm .08$ & $.68 \pm .20$ \\
{\tt optdigits} & 5620 & 64 & $.88 \pm .00$ & $.88 \pm .00$ \\
{\tt park} & 3000 &  20 & $.58 \pm .06$ & $.61 \pm .08$ \\
{\tt thyroid} & 3163 & 5 & $.77 \pm .00$ & $.77 \pm .00$ \\
{\tt splice} & 1000 & 60 & $.63 \pm .01$ & $.66 \pm .03$ \\
{\tt wine} & 6497 & 11 & $.63 \pm .10$ & $.69 \pm .13$ \\
\end{tabular}
}
\end{center}
\vspace{-0.3cm}
\caption{
\small
Size of dataset ($m$), features ($d$) and, the overall fraction of
remaining features in the training set after data-independent ($F_I$)
or data-dependent ($F_D$) corruption.}
\label{table:data}
\vspace{-0.5cm}
\end{table}
We report average error and standard deviations over 5 trials, using
1000 random training examples and corruption patterns. We tune
hyper-parameters using a grid search from $2^{-12}$ to $2^{10}$.
Further details and explicit corruption processes appear in the
appendix.

\subsection{\mbox{Online Corruption Dependent Hypothesis}}

Here we analyze the online algorithm presented in section
\ref{sec:matrix-alg} using two different types of regularization.  The
first method simply penalizes the Frobenius norm of the parameter matrix
$\A$ (frob-reg), $\cR(\A) = \| \A \|_F^2$.
The second method (sparse-reg) forces a sparse solution by
constraining many entries of the parameter matrix equal to zero as
mentioned in Section~\ref{sec:failed_sensors}. We use the regularizer
$\cR(\A) = \gamma \|\A \1\|^2 + \|\A\|_F^2$, where $\gamma$ is an
additional tunable parameter. This choice of regularization is based
on the example given in equation (\ref{eqn:Amatrixlocal}), where we
would have $\|\A\1\| = \|\w^*\|$.

We apply these methods to the {\tt splice} classification task and the
{\tt optdigits} dataset in several one vs.\ all classification tasks.
For {\tt splice}, the sparsity pattern used by the sparse-reg method
is chosen by constraining those entries $[\A]_{i,j}$ where feature $i$
and $j$ have a correlation coefficient less than 0.2, as measured with
the corrupted training sample.  In the case of {\tt optdigits}, only
entries corresponding to neighboring pixels are allowed to be
non-zero.  

Figure \ref{fig:plots} shows that, when subject to data-independent
corruption, the zero imputation, mean imputation and frob-reg methods
all perform relatively poorly while the sparse-reg method provides
significant improvement for the {\tt splice} dataset. 
Furthermore, we find data-dependent corruption is quite harmful to
mean imputation as might be expected, while both frob-reg and
sparse-reg still provide significant improvement over zero-imputation.
More surprisingly, these methods also perform better than training on
uncorrupted data. We attribute this to the fact that we are using a
richer hypothesis function that is parametrized by the corruption
vector while the standard algorithm uses only a fixed hypothesis.  In
Table \ref{table:online-digits} we see that the sparse-reg performs at
least as well as both zero and mean imputation in all tasks and
offers significant improvement in the 3-vs-all and 6-vs-all task. In
this case, the frob-reg method performs comparably to sparse-reg and
is omitted from the table due to space.

\begin{figure}
\begin{center}
\begin{tabular}{cc}
\includegraphics[width=0.46\columnwidth]{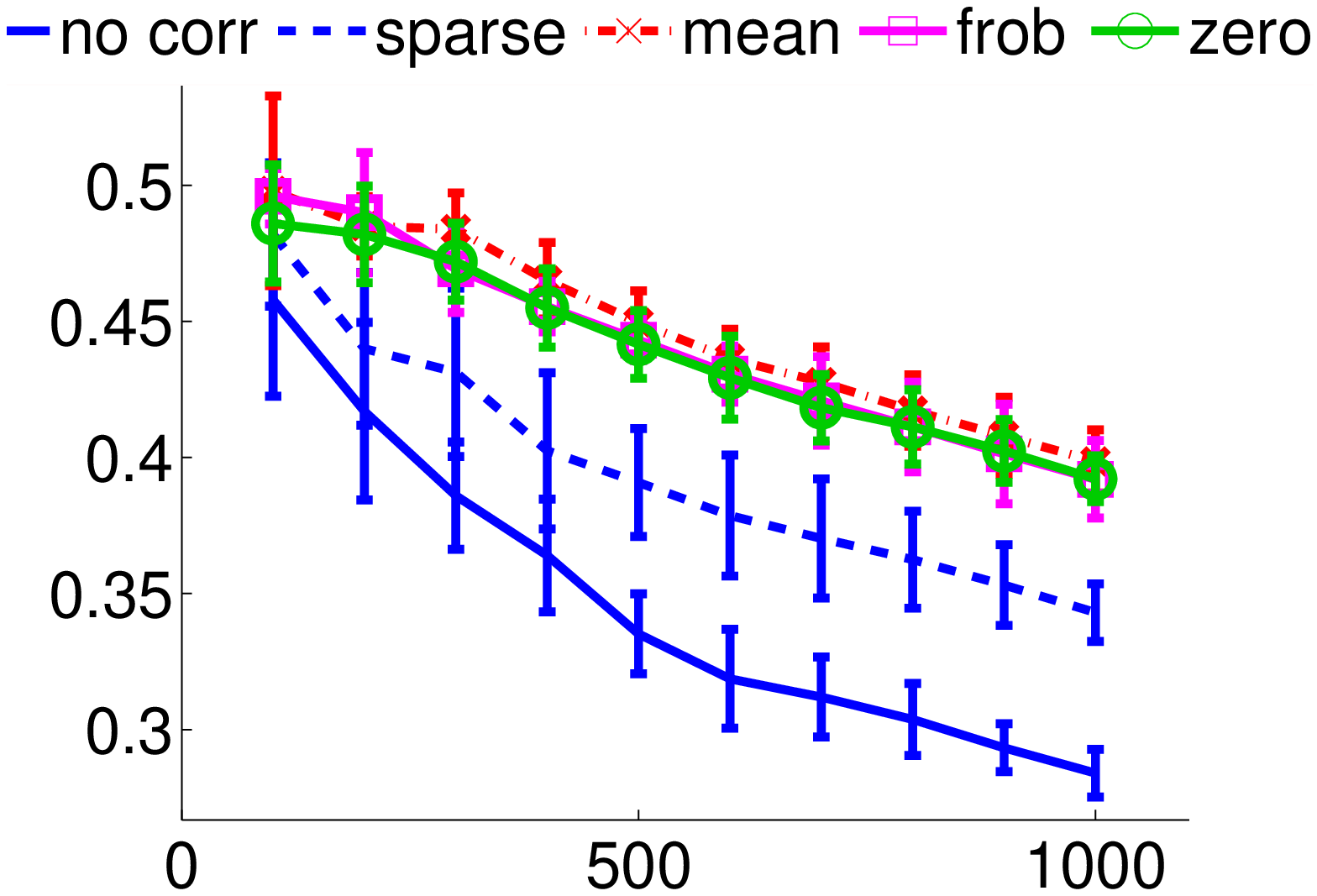} &
\hs
\includegraphics[width=0.46\columnwidth]{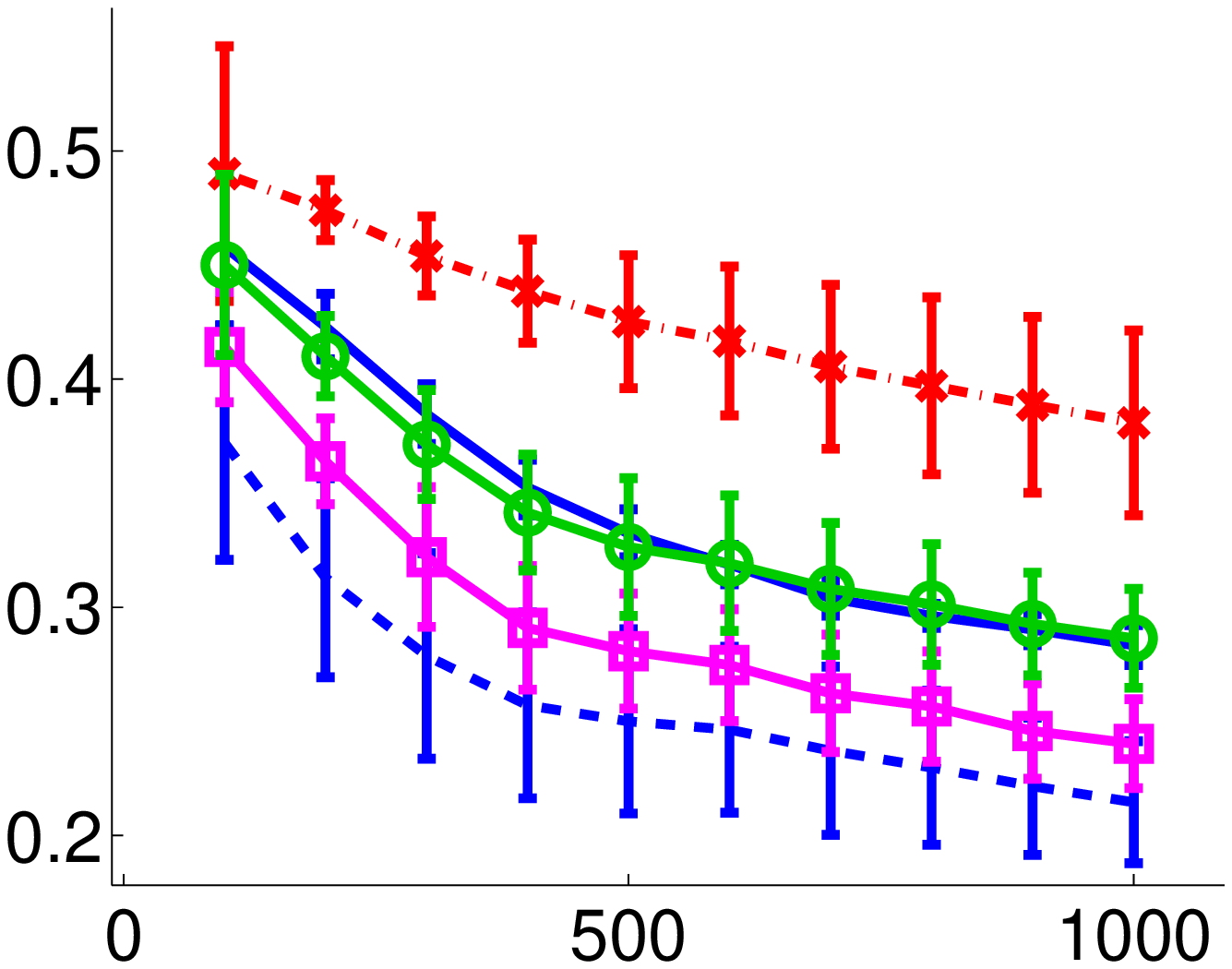} \\
\includegraphics[width=0.46\columnwidth]{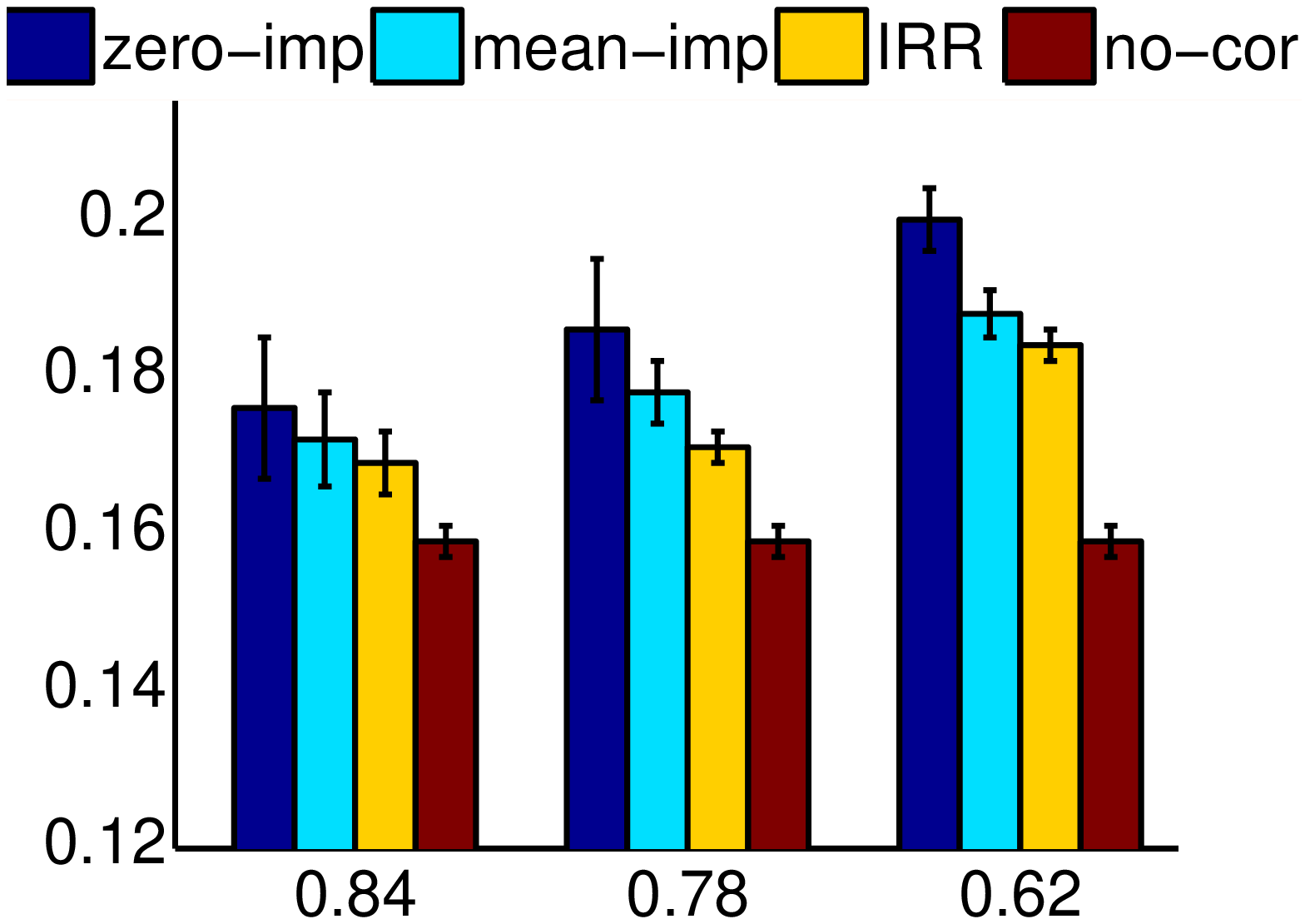} &
\hs
\includegraphics[width=0.46\columnwidth]{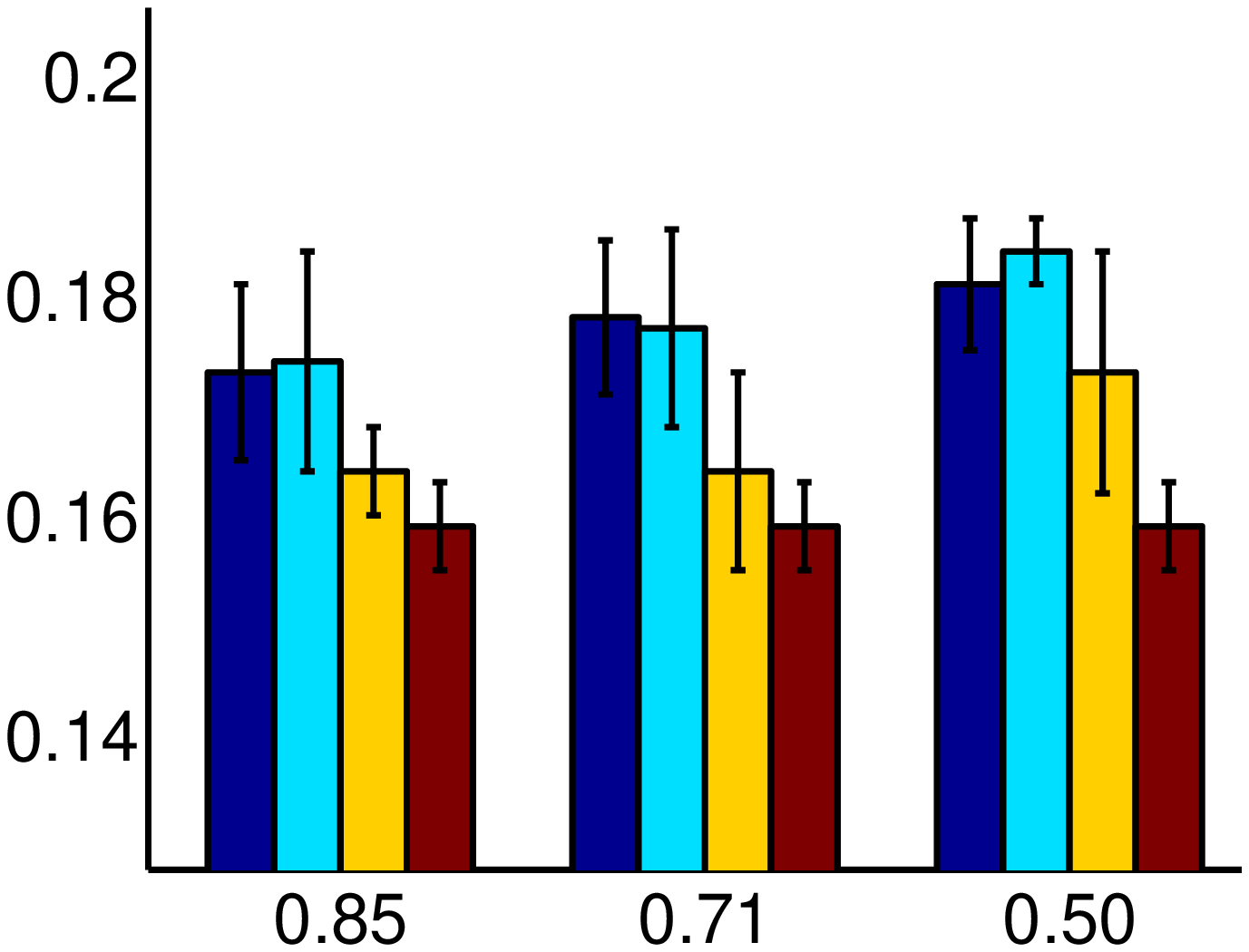} 
\vspace{-0.35cm}
\end{tabular}
\end{center}
\vspace{-0.25cm}
\caption{\small
0/1 loss as a function of $T$ for {\tt splice} dataset with
independent (top left) and dependent corruption (top right).  RMSE on
{\tt abalone} across varying amounts of independent (bottom left) and
dependent corruption (bottom right); fraction of features remaining
indicated on x-axis.}
\label{fig:plots}
\vspace{-0.2cm}
\end{figure}

\begin{table}
\begin{center}
{\small
\begin{tabular}{l|ccc|c}
 & \hhs zero-imp & \hs mean-imp & \hs sparse-reg & no corr \\
\hline
\hline
{2}
 & \hhs $.035 \pm .002$ & \hs $.039 \pm .004$ & \hs $.033
\pm .003$ &  $.024 \pm .002$ \\
\hline
{3}
 & \hhs $.041 \pm .002$ & \hs $.043 \pm .001$ & \hs $\mathbf{.039
\pm .002}$ & $.027 \pm .003$ \\
\hline
{4}
 & \hhs $.020 \pm .002$ & \hs $.023 \pm .002$ & \hs $.021
\pm .001$ & $.015 \pm .001$ \\
\hline
{6}
 & \hhs $.026 \pm .002$ & \hs $.024 \pm .002$ & \hs $\mathbf{.023
\pm .002}$ & $.015 \pm .002$
\end{tabular}
}
\end{center}
\vspace{-0.4cm}
\caption{\small
One-vs-all classification results on \texttt{optdigits} dataset
(target digit in first column) with column-based
corruption for 0/1 loss.}
\label{table:online-digits}
\vspace{-0.5cm}
\end{table}

\subsection{Imputed Ridge Regression}

In this section we consider the performance of IRR across
many datasets.  We found standard SDP solvers to be quite slow for
problem~(\ref{irr_relaxed}). We instead use a semi-infinite
linear program (SILP) to find an approximately optimal
solution (see e.g.~\cite{SDP-SILP} for details).  

In Tables~\ref{table:imputed-results-indep}
and~\ref{table:imputed-results-dependent} we compare the performance
of the IRR algorithm to zero and mean imputation as well as to
standard ridge regression performance on the uncorrupted data.  Here
we see IRR provides improvement over zero-imputation in all cases and
does at least as well as mean-imputation when dealing with
data-independent corruption. For data-dependent corruption, IRR
continues to perform well, while mean-imputation suffers.  For this
setting, we have also compared to an \emph{independent-imputation}
method, which imputes data using an $\M$ matrix that is trained
independently of the learning algorithm.  In particular the $i_{th}$
column of $\M$ is selected as the best linear predictor of the
$i_{th}$ feature given the rest, i.e.\ the solution to: $\argmin_\v
\sum_{k \in \cX_i} ([\x_k]_i - \sum_{j \neq i} [\x_k]_j [\v]_j)^2,$
where $\cX_i$ is the set of training examples that have the $i_{th}$
feature present.  Although, this method can perform better than
mean-imputation, the joint optimization solution provided by IRR
provides an even more significant improvement.  At the bottom of Table
\ref{table:imputed-results-dependent} we also measure performance with
    {\tt thyroid} which has naturally missing values.  Here again IRR
    performs significantly better than the competitor methods.
    Zero-imputation is not shown due to space, but it performs
    uniformly worse. Figure \ref{fig:plots} shows more detailed
    results for the {\tt abalone} dataset across different levels of
    corruption and displays the consistent improvement which the IRR
    algorithm provides.

\begin{table}
\begin{center}
{\small
\begin{tabular}{l|ccc|c}
& \hhs zero-imp & \hs mean-imp & \hs IRR & no corr \\
\hline
\hline
A & \hhs $.199 \pm .004$ & \hs $.187 \pm .003$ & \hs $\mathbf{.183 \pm
.002}$ & $.158 \pm .002$ \\
H & \hhs $.414 \pm .025$ & \hs $\mathbf{.370 \pm .019}$ & \hs
$\mathbf{.373 \pm .019}$ & $.288 \pm .001$ \\
P & \hhs $.457 \pm .006$ & \hs $\mathbf{.445 \pm .004}$ & \hs $.451 \pm .004$ & $.422 \pm .004$ \\
W & \hhs $.280 \pm .006$ & \hs $\mathbf{.268 \pm .009}$ & \hs
$\mathbf{.269 \pm .008}$ & $.246 \pm .001$ \\
\end{tabular}
}
\caption{RMSE for various imputation methods across the datasets {\tt
    abalone} (A), {\tt housing} (H), {\tt park} (P) and {\tt wine} (W)
  when subject to data-independent corruption}
\label{table:imputed-results-indep}
\end{center}
\end{table}
\begin{table}
\begin{center}
{\small
\begin{tabular}{l|ccc|c}
& \hhs mean-imp & \hs ind-imp & \hs IRR & no corr \\
\hline
\hline
A & \hhs $.180 \pm .006$ & \hs $.183 \pm .012$ & \hs $\mathbf{.167 \pm
.011}$ & $.159 \pm .004$ \\
H & \hhs $.400 \pm .064$ & \hs $.363 \pm .041$ & \hs $\mathbf{.326 \pm
.035}$ & $.289 \pm .001$ \\
P & \hhs $.444 \pm .008$ & \hs $.423 \pm .015$ & \hs $\mathbf{.377 \pm
.035}$ & $.422 \pm .001$ \\
W & \hhs $.264 \pm .009$ & \hs $.260 \pm .011$ & \hs $.256 \pm .011$ & $.247
\pm .001$\\
\hline \hline
T & \hhs $.531 \pm .005$ & \hs $.528 \pm .003$ & \hs $\mathbf{.521 \pm
.004}$ & \hs -- 
\end{tabular}
}

\end{center}
\vspace{-0.3cm}
\caption{\small RMSE for various imputation methods across the
  datasets {\tt abalone} (A), {\tt housing} (H), {\tt park} (P) and
  {\tt wine} (W) when subject to data-dependent corruption. The {\tt
    thyroid} (T) dataset has naturally occurring missing features.  }
\label{table:imputed-results-dependent}
\vspace{-0.1cm}
\end{table}


In Table \ref{table:impute-digits} we see that, with respect to the
column-corrupted {\tt optdigit} dataset, the IRR algorithm performs
significantly better than zero-imputation and mean-imputation in majority
of tasks. 
\begin{table}
\begin{center}
{\small
\begin{tabular}{l|ccc|c}
& zero-imp & mean-imp & IRR & no corr \\
\hline
\hline
2 & \hhs $.352 \pm .003$ & \hs $.351 \pm .004$ & \hs $\mathbf{.346 \pm
.002}$ & $.321 \pm .003$ \\
3 & \hhs $.450 \pm .005$ & \hs $.435 \pm .004$ & \hs $\mathbf{.426 \pm
.005}$ & $.398 \pm .004$ \\
4 & \hhs $.372 \pm .003$ & \hs $\mathbf{.363 \pm .002}$ & \hs
$\mathbf{.364 \pm .003}$ & $.345 \pm .002$ \\
6 & \hhs $.369 \pm .003$ & \hs $.360 \pm .002$ & \hs $\mathbf{.353 \pm
.003}$ & $.333 \pm .003$ \\
\end{tabular}
}
\end{center}
\vspace{-0.3cm}
\caption{\small RMSE (using binary labels) for one-vs-all
  classification on \texttt{optdigits} subject to column-based
  corruption. }
\label{table:impute-digits}
\vspace{-0.5cm}
\end{table}

%

\section{Conclusion}
We have introduced two new algorithms, addressing the problem of
learning with missing features in both the adversarial online and
i.i.d.\ batch settings.  The algorithms are motivated by intuitive
constructions and we also provide theoretical performance
guarantees.  Empirically we show encouraging initial results for
online matrix-based corruption-dependent hypotheses as well as many
significant results for the suggested IRR algorithm, which indicate
superior performance when compared to several baseline imputation
methods.  

\subsubsection*{Acknowledgements}
We gratefully acknowledge the support of the NSF under award
DMS-0830410. AA was partially supported by an MSR PhD Fellowship. We
also thank anonymous reviewers for suggesting additional references
and improvements to proofs.

\bibliographystyle{plain}
{ \small
\bibliography{corrupted_features}

\begin{thebibliography}{10}

\bibitem{AbernethyABR2009minimax}
J.~Abernethy, A.~Agarwal, P.~L. Bartlett, and A.~Rakhlin.
\newblock A stochastic view of optimal regret through minimax duality.
\newblock {\em CoRR}, abs/0903.5328, 2009.

\bibitem{bm_rademacher}
P.L. Bartlett and S.~Mendelson.
\newblock {Rademacher and Gaussian complexities: Risk bounds and structural
  results}.
\newblock {\em JMLR}, 3, 2003.

\bibitem{beck2003mirror}
A.~Beck and M.~Teboulle.
\newblock Mirror descent and nonlinear projected subgradient methods for convex
  optimization.
\newblock {\em Operations Research Letters}, 31(3), 2003.

\bibitem{CesaBianchiLugosi06book}
N.~Cesa-Bianchi and G.~Lugosi.
\newblock {\em Prediction, Learning, and Games}.
\newblock Cambr. Univ. Press, 2006.

\bibitem{cesa_efficient}
N.~Cesa-Bianchi, S.~Shalev-Shwartz, and O.~Shamir.
\newblock {Efficient learning with partially observed attributes}.
\newblock {\em ICML}, 2010.

\bibitem{cesa-noise}
N.~Cesa-Bianchi, S.S. Shwartz, and O.~Shamir.
\newblock {Online Learning of Noisy Data with Kernels}.
\newblock {\em COLT}, 2010.

\bibitem{chechik_struct}
G.~Chechik, G.~Heitz, G.~Elidan, P.~Abbeel, and D.~Koller.
\newblock {Max-margin classification of data with absent features}.
\newblock {\em JMLR}, 9, 2008.

\bibitem{CoverOr96}
T.M. Cover and E.~Ordentlich.
\newblock Universal portfolios with side information.
\newblock {\em Information Theory, IEEE Transactions on}, 42(2):348 --363, mar
  1996.

\bibitem{dekel_corrupted}
O.~Dekel, O.~Shamir, and L.~Xiao.
\newblock {Learning to classify with missing and corrupted features}.
\newblock {\em Machine learning}, 2010.

\bibitem{dempster_em}
A.P. Dempster, N.M. Laird, and D.B. Rubin.
\newblock {Maximum likelihood from incomplete data via the EM algorithm}.
\newblock {\em Journ. of the Royal Stat. Society}, 39(1), 1977.

\bibitem{Globerson2006nightmare}
A.~Globerson and S.~Roweis.
\newblock Nightmare at test time: robust learning by feature deletion.
\newblock In {\em ICML}, 2006.

\bibitem{HazanMe2007}
E.~Hazan and N.~Megiddo.
\newblock Online learning with prior information.
\newblock In {\em COLT}, 2007.

\bibitem{SDP-SILP}
K.~Krishnan and J.E. Mitchell.
\newblock {Semi-infinite linear programming approaches to semidefinite
  programming problems}.
\newblock {\em Novel approaches to hard discrete optimization problems}, 37,
  2003.

\bibitem{little_rubin}
R.J.A. Little and D.B. Rubin.
\newblock {\em {Statistical analysis with missing data}}.
\newblock Wiley New York, 1987.

\bibitem{marlin}
B.~M. Marlin.
\newblock {\em Missing Data Problems in Machine Learning}.
\newblock PhD thesis, University of Toronto, 2008.

\bibitem{yudin83book}
A.~S. Nemirovski and D.~B. Yudin.
\newblock {\em Problem Complexity and Method Efficiency in Optimization}.
\newblock 1983.

\bibitem{long_version}
A.~{Rostamizadeh}, A.~{Agarwal}, and P.~{Bartlett}.
\newblock {Online and Batch Learning Algorithms for Data with Missing
  Features}.
\newblock {\em ArXiv e-prints}, 2011.

\end{thebibliography}
}

\newpage

\appendix
\section{Proof of Proposition~\ref{prop:lower-bound}}

The strategy used here is to consider the total regret accumulated by an
algorithm on several different tasks, each one indexed by a different
$\z_t$.

First, in order to simplify the interaction between $\z_t$ and $\x_t$,
suppose only the first $d/2$ coordinates of $\x_t$ contain information
(and the rest are always set equal to 0) and assume only the last
$d/2$ coordinates of $\z_t$ contain any information (the rest are
always set equal to 1).  Thus, for every one of the $2^{d/2}$ distinct
values of $\z_t$ we associate a different independent $\w^*(\z_t)$, or
task, which the algorithm is trying to learn.

The main intuition is that the learning problem now reduces to a
multitask classification problem with $K = 2^{d/2}$ different
tasks. Without further assumptions, it can be shown that the minimax
regret of such a multitask classification problem is
as bad as solving the tasks independently. 

We partition the total number of iterations $T = \sum_{i=1}^{2^{d/2}}
T_i$, where each $T_i$ is the number of iterations a particular $\z_t$
was used by the adversary. In order to analyze the minimax regret, we
can use von-Neumann duality (see e.g. \cite{AbernethyABR2009minimax})
to get
\begin{align*}
  &\inf_{\w_1}\sup_{(\x_1,\z_1,y_1)}\cdots\inf_{\w_T}\sup_{(\x_T,\z_T,y_T)}\\
   & \left[\sum_{t=1}^T
    \ell(\dprod{\w_t(\z_t)}{\x_t'},y_t) - \inf_{\w \in \cW} \sum_{t=1}^T
    \ell(\dprod{\w(\z_t)}{\x_t'},y_t)\right] \\&=
  \sup_{\mathbf{p}}\E\left[\sum_{t=1}^T \inf_{\w_t \in \cW}
    \E[\ell(\dprod{\w_t(\z_t)}{\x_t'},y_t) | (\x_s,y_s,\z_s)_1^{t-1}]
\right. \\ 
    & \qquad \qquad - \left. \inf_{\w \in \cW}\sum_{t=1}^T
\ell(\dprod{\w(\z_t)}{\x_t'},y_t)\right],
\end{align*}
where the supremum is over joint distributions on sequences
$(\x_1,y_1,\z_1),\ldots,(\x_T,y_T,\z_T)$.

It is clear that the first term decomposes over the $2^{d/2}$ tasks
(since it decomposes over individual examples). The second
minimization optimizes over all mappings in the set $\cW$. This can be
done alternatively by maximizing over the choice of a weight vector
for each task individually. As a result, the minimax regret decomposes as a
sum of the minimax regrets for each task.

If we choose $T_i = T / 2^{d/2}$ then the total
regret (which is the sum of the regrets accumulated from each task) 
is measured as follows,
\begin{equation*}
  \sum_{i=1}^{2^{d/2}} R^*(T_i, \ell)
  = \sum_{i=1}^{2^{d/2}} R^*\left(\frac{T}{2^{d/2}},\ell\right)
  = 2^{d/2} R^*\left(\frac{T}{2^{d/2}},\ell\right).
\end{equation*}
This completes the proof of the proposition.

\section{Proof of Theorem~\ref{thm:mmdbound}}

The proof is standard and just included for completeness. We recall
from the update rule that
\begin{equation*}
  \A_{t+1} = \arg\min_{\A \in \cA}\left\{\eta_t\dprod{\nabla
    \ell_t(\A_t)}{\A}_F + D_{\cR}(\A, \A_t)\right\}.
\end{equation*}
Consequently, $\A_{t+1}$ satisfies the first order optimality
conditions:
\begin{equation}
\dprod{\eta_t\nabla\ell_t(\A_t) + \nabla \cR(\A_{t+1}) -
  \nabla\cR(\A_t)}{\A - \A_{t+1}}_F \geq 0,
\label{eqn:firstorder}
\end{equation}
for all $\A \in \cA$. Now for any fixed $\A \in \cA$, we can write the
regret 
\begin{align}
  \nonumber&\sum_{t=1}^T\ell_t(\A_t) - \ell_t(\A) \leq
  \sum_{t=1}^T\dprod{\nabla\ell_t(\A_t)}{\A_t - \A}_F\\
  \nonumber&\leq \sum_{t=1}^T\left[\dprod{\nabla\ell_t(\A_t)}{\A_{t+1} - \A}_F +
  \dprod{\nabla\ell_t(\A_t)}{\A_t - \A_{t+1}}_F\right]\\
  \nonumber&\leq \sum_{t=1}^T\left[\frac{1}{\eta_t}\dprod{\nabla\cR(\A_{t+1}) -
      \nabla\cR(\A_t)}{\A - \A_{t+1}}_F \right.\\&\left.\qquad\qquad\qquad\qquad+
    \dprod{\nabla\ell_t(\A_t)}{\A_t - \A_{t+1}}_F\right].
  \label{eqn:firstbound}
\end{align}
Here the first inequality follows from the convexity of the loss
$\ell_t$ and the last inequality is a consequence
of~\eqref{eqn:firstorder}. Also, applying~\eqref{eqn:firstorder} with
$\A = \A_t$ gives
\begin{align*}
  &\eta_t\dprod{\nabla\ell_t(\A_t)}{\A_t - \A_{t+1}}_F \\&\geq
  \dprod{\nabla\cR(\A_{t+1}) - \nabla\cR(\A_t)}{\A_{t+1} - \A_t}_F\\
  &= D_{\cR}(\A_t, \A_{t+1}) + D_{\cR}(\A_{t+1}, \A_t)\\
  &\geq \|\A_t - \A_{t+1}\|^2,
\end{align*}
where the last step is a consequence of the strong convexity of the
regularizer $\cR$. Finally, applying H\"older's inequality to the LHS
of the above display yields
\begin{align*}
  \|\A_t - \A_{t+1}\|^2 &\leq \eta_t\dprod{\nabla\ell_t(\A_t)}{\A_t -
    \A_{t+1}}_F\\
  &\leq \eta_t\|\nabla\ell_t(\A_t)\|_*\|\A_t - \A_{t+1}\|,
\end{align*}
where $\|\cdot\|_*$ is the dual norm to $\|\cdot\|$. Hence we get
\begin{equation}
  \|\A_t - \A_{t+1}\| \leq \eta_tG,
  \label{eqn:smallstepbound}
\end{equation}
where the last step follows from the Lipschitz assumption in the
theorem statement. As a result, we can bound the second term
in~\eqref{eqn:firstbound} as 
\begin{align}
  \nonumber\dprod{\nabla\ell_t(\A_t)}{\A_t - \A_{t+1}}_F &\leq
  \|\nabla\ell_t(\A_t)\|_*\|\A_t - \A_{t+1}\|\\
  &\leq \eta_tG^2. 
  \label{eqn:term2bound}
\end{align}

For the first term in~\eqref{eqn:firstbound}, we note that 
\begin{align*}
  &\dprod{\nabla\cR(\A_{t+1}) - \nabla\cR(\A_t)}{\A_{t+1} - \A_t}_F\\
  &\leq D_{\cR}(\A, \A_t) - D_{\cR}(\A_{t+1}, \A_t) - D_{\cR}(\A,
  \A_{t+1})\\
  &\leq D_{\cR}(\A, \A_t)  - D_{\cR}(\A,\A_{t+1}),
\end{align*}
where the last step follows from non-negativity of Bregman
divergences. Finally, we combine the two bounds from above and
substitute for the value of $\eta_t = R/(G \sqrt{T})$. Simplifying yields
the statement of the theorem.

\section{Proof of Proposition~\ref{prop:irr_relaxed}}

In order to formulate a tractable problem we first rewrite the
imputed ridge regression problem in its dual formulation.
\begin{align*}
 & \min_{\M} \max_\Alpha 
  ~~  2 \sum_{i=1}^T \alpha_i y_i  - \\
 & ~~ \sum_{i,j=1}^T \! \alpha_i \alpha_j \big( (\x_i + \Zb_i
    \M^\top\x_i)^\top (\x_j + \Zb_j \M^\top\x_j) + \lambda T \I \big) \\
& \mathrm{s.t.} ~~ \|\M\|_F^2 \leq \gamma^2
\end{align*}
The inner maximization problem is concave in $\Alpha$ and the
optimal solution for any fixed $\M$ is found via the standard closed
form solution for ridge regression:
\begin{equation*}
  \Alpha^* = (\underbrace{(\X + \Zb \circ \M\X) (\X + \Zb \circ
\M\X)^\top}_{\K_\M} + \lambda T \I)^{-1} \y \,,
\end{equation*}
where $\circ$ denotes the component-wise (Hadamard) product between
matrices and $\K_\M$ will be used to denote the Gram matrix containing
dot-products between imputed training instances.  Plugging this
solution into the minimax problem results in the following matrix
fractional minimization problem,
\begin{equation*}
  \min_\M ~ \y (\K_\M + \lambda T \I)^{-1} \y,
  ~~ \mathrm{s.t.} ~ \|\M\|_F^2 \leq \gamma^2 \,.
\end{equation*}
This problem is still not convex in $\M$ due to the quadratic terms
that appear in $\K_\M$.  The main idea for the convex relation will be
to introduce new variables $[\N_k]_{i,j}$ which substitute the
quadratic terms $[\M]_{i,k} [\M]_{j,k}$, resulting in a matrix
$\K_{\M\N}$ that is linear in terms of the optimization variables $\M$
and $\N_k$. This is shown precisely below:
\begin{align*}
[\K_{\M}]_{i,j} & = \x_i^\top \x_j^\top 
         + \x_i^\top \M \Zb_i \x_j
         + \x_i^\top \Zb_j \M^\top \x_j \\
   & ~~ + \underbrace{\x_i^\top \M \Zb_i \Zb_j \M^\top \x_j}_{
   \sum_{r,s,k=1}^d [\x_i]_r [\x_j]_s [\zb_i]_k [\zb_j]_k [\M]_{r,k}
    [\M]_{s,k}} \\
[\K_{\M\N}]_{i,j} & = \x_i^\top \x_j^\top 
         + \x_i^\top \M \Zb_i \x_j
         + \x_i^\top \Zb_j \M^\top \x_j \\
    & ~~ + \underbrace{\sum_{k=1}^d [\zb_i]_k [\zb_j]_k \x_i^\top \N_k
\x_j}_{\sum_{r,s,k=1}^d [\x_i]_r [\x_j]_s [\zb_i]_k [\zb_j]_k
[\N_k]_{r,s}}
\end{align*}
Note that the matrix $\K_{\M\N}$ no longer necessarily corresponds to
a Gram matrix and that $(\K_{\M\N} + \lambda T \I)$ may no longer be
positive semi-definite (which is required for the convexity of a
matrix fractional problem objective).  Thus, we add an additional
explicit positive semi-definiteness constraint resulting in the
following optimization problem,
\begin{align*}
  \min_{\M, \N, t} & ~~ t \\
  \mathrm{s.t.} 
    & ~~ t - \y^\top (\K_{\M\N} + \lambda T \I)^{-1} \y \geq 0 \\ 
    & ~~ \K_{\M\N} \succeq 0 \\
    & ~~ \|\M\|_F^2 \leq \gamma^2, 
      ~ \sum_{k=1}^d \|\N_k\|_F^2 \leq \gamma^4 \,,
\end{align*}
where we've additionally added the dummy variable $t$ and also
constrained the norm of the new $[\N_k]_{i,j}$ variables.  The choice
of the upper bound is made with the knowledge that $[\N_k]_{i,j}$
replaces the variables $[\M]_{i,k} [\M]_{j,k}$ and that the bound
$\|\M\|_F \leq \gamma$ implies $\sum_{i,j,k=1}^d [\M]_{i,k}^2
[\M]_{j,k}^2 = \sum_{k=1}^d (\sum_{i=1}^d [\M]_{i,k}^2)^2 
\leq (\sum_{i,k=1}^d [\M]_{i,k}^2)^2 \leq \gamma^4$.

The constraint involving the dummy variable $t$ is a Schur complement
and can be replaced with an equivalent positive semi-definiteness
constraint, which results in a standard form semidefinite program
and completes the proof.

\section{Proof of Theorem~\ref{thm:rademacher}}

It suffices to bound each of the following terms individually:
\begin{align*}
 (a) & ~~ \E_\Ssigma \bigg[ \sup_{\Alpha,\M,\N} 
    \Big| \sum_{i,j=1}^T \sigma_i \alpha_j {\x'_i}^\top \x_j
    \Big| \bigg], \\
 (b) & ~~ \E_\Ssigma \bigg[ \sup_{\Alpha,\M,\N} \Big| \sum_{i,j=1}^T 
    \sigma_i \alpha_j {\x'_i}^\top \Zb_j \M^\top \x_j
    \Big| \bigg], \\
 (c) & ~~ \E_\Ssigma \bigg[ \sup_{\Alpha,\M,\N} \Big| \sum_{i,j=1}^T 
    \sigma_i \alpha_j {\x'_i}^\top \M \Zb'_i \x_j 
    \Big| \bigg], \\
 (d) & ~~ \E_\Ssigma \bigg[ \sup_{\Alpha,\M,\N} \Big| \sum_{i,j=1}^T 
    \sigma_i \alpha_j \sum_{k=1}^d [\zb'_i]_k [\zb_j]_k {\x'}_i^\top \N_k
    \x_j
    \Big| \bigg] ,
\end{align*}
then combining the bounds and dividing by $T$ proves the theorem.

To bound $(a)$ we  first apply the Cauchy-Schwarz inequality to
separate the $\sigma_i$ and $\alpha_j$ terms:
\begin{multline*}
  \E_{\bf \sigma} \big[ \sup_{\Alpha} |\sum_{i,j=1}^T \sigma_i
     \alpha_j {\x'_i}^\top \x_j| \big] \\
   \leq \sup_\Alpha ||\sum_{i=1}^T \alpha_i \x_i|| \E_\sigma \big[
   ||\sum_{i=1}^T \sigma_i \x'_i|| \big]
\end{multline*}
The $\alpha_i$ term is bounded as follows,
\begin{align*}
  \sup_\Alpha ||\sum_{i=1}^T \alpha_i \x_i||
  & = \sup_\Alpha \sqrt{\Alpha^\top \X \X^\top \Alpha} \\
  & \leq \sup_\Alpha \|\Alpha\| \|\X \X^\top \|_2^{1/2} \\
  & \leq \frac{B}{\lambda \sqrt{T}} \sqrt{\Tr[\X \X^\top]}
  \leq \frac{B R}{\lambda} \,.
\end{align*}
The $\sigma_i$ term is bounded using the fact that for Rademacher
independent variables $\sigma_i$ and $\sigma_j$ the expectation
$\E[\sigma_i \sigma_j] = 0$.
\begin{align*}
  \E_{\Ssigma} \big[ ||\sum_{i=1}^T \sigma_i \x'_i|| \big]
  & = \E_{\Ssigma} \big[ \sqrt{\Ssigma^\top \X' \X'^\top \Ssigma} \big] \\
  & = \sqrt{\Tr[\X' \X'^\top]} \leq R \sqrt{T} \,.
\end{align*}
Thus, the first term is bounded as $(a) \leq \frac{BR^2
\sqrt{T}}{\lambda}$. To bound the second term, $(b)$, we again apply
Cauchy-Schwarz to separate the $\Ssigma$ and $(\Alpha,\M)$ terms.  The
$\Ssigma$ portion is again bounded by $R \sqrt{T}$ and the remainder
of the bound follows similar steps as in the bound of $(a)$
decomposing the norm into a product between $\|\Alpha\|$ and a trace
term. If we define $[\B]_{i,j} = \x_i \M \Zb_i \Zb_j \M^\top \x_j$, then
\begin{multline}
\label{eq:alpha_split}
  \sup_{\Alpha, \M} \|\sum_{i=1}^T \alpha_i \Zb_i \M^\top \x_i\| 
  = \sup_{\Alpha, \M} \left( \Alpha^\top \B \Alpha \right)^{1/2} \\
  \leq \sup_{\Alpha, \M} \|\Alpha\| \lambda_{\max}(\B)^{1/2}
  \leq \sup_{\Alpha, \M} \|\Alpha\| \Tr(\B)^{1/2} \\
  \leq \sup_{\Alpha, \M} \|\Alpha\| \Big(\sum_{i=1}^T 
    \x_i^\top \M \Zb_i \Zb_i \M^\top \x_i \Big)^{1/2} \,,
\end{multline}
Where the second inequality follows from the fact that $\B$ is
positive semi-definite.
Note, since $\Zb_i$ is a diagonal $0/1$ matrix we have $\Zb_i \Zb_i =
\Zb_i$.  To bound the term depending on $\M$ we use the following set
of inequalities
\begin{align*}
  \sup_\M \Big( \x_i^\top \M \Zb_i \M^\top \x_i \Big)^{1/2}
  & \leq \sup_\M \Big( \|\x_i\|^2 \Tr[\M \Zb_i \M^\top] \Big)^{1/2}
\\
  & \leq \sup_\M R \Big( \dprod{\M}{\M\Zb_i}_F \Big)^{1/2}
\\
  & \leq \sup_\M R \Big( \|\M\|_F \|\M\Zb_i\|_F \Big)^{1/2}
\\
  & \leq \sup_\M R \|\M\|_F \leq \gamma R \,.
\end{align*}
Thus, the final bound on the second term is $(b) \leq \frac{\gamma BR^2
\sqrt{T}}{\lambda}$.  
In order to separate the $\Ssigma$ terms from
the $(\Alpha, \M)$ terms in the third term, $(c)$, the expression is first
expanded, using $[\M]_{:,s}$ to denote the $s_{th}$ column of the matrix
$\M$, and then Cauchy-Schwarz is applied,
\begin{align*}
  & \sum_{i,j=1}^T \sigma_i \alpha_j {\x'_i}^\top \M \Zb'_i \x_j  \\
  & = \sum_{i,j=1}^T \sigma_i \alpha_j {\x'_i}^\top \big( \sum_{s=1}^d
    [\M]_{:,s} [\zb'_i]_s [\x_j ]_s \big) \\
  & = \sum_{s=1}^d \Big(\sum_{i=1}^T  \sigma_i [\zb'_i]_s \x'_i \Big)^\top
  \Big(\sum_{j=1}^T \alpha_j [\x_j]_s [\M]_{:,s} \Big) \\
  & \leq \underbrace{\Big( \sum_{s=1}^d \Big\| \sum_{i=1}^T  \sigma_i [\zb'_i]_s
\x'_i \Big\|^2 \Big)^{1/2}}_{\mathrm{(i)}}
  \underbrace{\Big(\sum_{s=1}^d \Big\| \sum_{j=1}^T \alpha_j [\x_j]_s [\M]_{:,s}
  \Big\|^2 \Big)^{1/2}}_{\mathrm{(ii)}} \\
\end{align*}
The inequality follows from the fact that, given vectors $\v_s, \u_s$,
we have: 
\begin{multline*}
\sum_s \v_s^\top \u_s = \sum_{s,r} [\v_s]_r [\u_s]_r
\leq ( \sum_{s,r} [\v_s]_r^2 \sum_{s,r} [\u_s]_r^2)^{1/2} \\
= ( \sum_{s} \|\v_s\|^2)^{1/2} (\sum_{s} \|\u_s\|^2)^{1/2} \,,
\end{multline*}
where the inequality follow from Cauchy-Schwarz.
To bound the term (i) that depends on $\Ssigma$ we note
\begin{multline*}
  \E_\Ssigma \Big[ \Big( \sum_{s=1}^d \Big\| \sum_{i=1}^T  \sigma_i
    [\zb'_i]_s \x'_i \Big\|^2 \Big)^{1/2} \Big] \\
  \leq \Big( \sum_{s=1}^d \E_\Ssigma \Big[ \Big\| \sum_{i=1}^T
    \sigma_i [\zb'_i]_s \x'_i \Big\|^2 \Big] \Big)^{1/2}
\end{multline*}
and for any $s$ bound the expectation as follows,
\begin{align*}
 \E_\Ssigma \Big[ \Big\| \sum_{i=1}^T \sigma_i [\zb'_i]_s \x'_i \Big\|^2 \Big]
 = \sum_{i=1}^T [\zb_i']_s^2 \| \x'_i \|^2 \leq T R^2 \,.
\end{align*}
Thus, the term (i) is bounded by $R \sqrt{dT}$.
To bound (ii) we again first rewrite the expression in terms of
$\|\Alpha\|$ and a matrix trace (using a similar argument as in
\eqref{eq:alpha_split}, but with $[\B]_{i,j} = \sum_{s=1}^d [\x_i]_s
[\x_j]_s \sum_{r=1}^d [\M]_{r,s}^2$), then apply the Cauchy-Schwarz
inequality,
\begin{align*}
  &\sup_{\Alpha,\M} \Big(\sum_{s=1}^d \Big\| \sum_{j=1}^T \alpha_j [\x_j]_s
      [\M]_{:,s} \Big\|^2 \Big)^{1/2}  \\
  & = \sup_{\Alpha,\M} \Big( \sum_{i,j=1}^T  \alpha_i \alpha_j
  \sum_{s=1}^d [\x_i]_s [\x_j]_s \sum_{r=1}^d [\M]_{r,s}^2 \Big)^{1/2} \\
  & \leq \sup_{\Alpha,\M} \|\Alpha\| \Big( \sum_{i=1}^T 
    \sum_{s=1}^d [\x_i]_s^2 \sum_{r=1}^d [\M]_{r,s}^2 \Big)^{1/2} \\
  & \leq \frac{B}{\lambda \sqrt{T}}
    \sup_{\M} \Big( \sum_{i=1}^T 
    (\sum_{s=1}^d [\x_i]_s^4)^{1/2} 
    (\sum_{s=1}^d (\sum_{r=1}^d [\M]_{r,s}^2)^2)^{1/2} \Big)^{1/2} \\
  & \leq \frac{B}{\lambda \sqrt{T}}
    \sup_{\M} \Big( \sum_{i=1}^T 
    (\sum_{s=1}^d [\x_i]_s^2)
    (\sum_{r,s=1}^d [\M]_{r,s}^2) \Big)^{1/2} \\
  & = \frac{B}{\lambda \sqrt{T}}
    \sup_{\M} \Big( \sum_{i=1}^T \|\x_i\|^2 \|\M\|_F^2 \Big)^{1/2}
  \leq \frac{\gamma B R}{\lambda} \,.
\end{align*}
Combining these two parts gives a bound of $(c) \leq \frac{\gamma B
R^2 \sqrt{dT}}{\lambda}$. Let $\V$ denote the matrix with $k_{th}$
column equal to $\N_k (\sum_{j=1}^T \alpha_j [\zb_j]_k \x_j)$, then
the bound on $(d)$ follows a similar pattern, first separating
$\Ssigma$ and $(\N, \Alpha)$ using the Cauchy-Schwarz inequality:  
\begin{align*}
  & \E_\Ssigma \bigg[ \sup_{\Alpha,\N} \Big| \sum_{i,j=1}^T 
    \sigma_i \alpha_j \sum_{k=1}^d [\zb'_i]_k [\zb_j]_k {\x'}_i^\top \N_k
    \x_j \Big| \bigg]  \\
  & = \E_\Ssigma \bigg[ \sup_{\Alpha,\N} \Big|
    \sum_{i=1}^T \sigma_i \x_i'^\top \V \zb'_i
    \Big| \bigg]  \\
  & = \E_\Ssigma \bigg[ \sup_{\Alpha,\N} \Big|
    \frob{\sum_{i=1}^T \sigma_i \x_i' \zb_i'^\top}{\V} 
    \Big| \bigg] \\
  & \leq \E_\Ssigma \Big[ \|\sum_{i=1}^T \sigma_i \x_i'
\zb_i'^\top\|_F \Big]  \sup_{\Alpha,\N} \| \V \|_F \,.
\end{align*}
The $\Ssigma$ term is bounded as follows,
\begin{multline*}
  \E_\Ssigma \Big[ \|\sum_{i=1}^T \sigma_i \x_i'
    \zb_i'^\top\|_F \Big] \\
  \leq \Big(  \E_\Ssigma \Big[  \sum_{i,j=1}^T \sigma_i \sigma_j
        \Tr \big[ \x_i' \zb_i'^\top \zb_j' \x_j'^\top  \big] \Big]
       \Big)^{1/2} \\
  =  \Big( \sum_{i=1}^T \|\x_i'\|^2 \| \zb_i'\|^2 \Big)^{1/2}
  \leq R \sqrt{dT} \,.
\end{multline*}
The supremum term is bounded as using the following set of
inequalities,
\begin{align*}
\sup_{\Alpha, \N} \| \V \|_F 
& = \sup_{\Alpha, \N}  \Big( \sum_{k=1}^d \| \sum_{i=1}^T  \alpha_i
  [\zb_i]_k \N_k \x _i \|^2 \Big)^{1/2}\\
& \leq \sup_{\Alpha, \N} \|\Alpha\| 
  \Big( \sum_{k=1}^d \sum_{i=1}^T [\zb_i]_k \x_i^\top \N_k^\top \N_k \x_i
  \Big)^{1/2} \\
& \leq \frac{B}{\lambda \sqrt{T}} \sup_\N \Big( \sum_{i=1}^T
  \|\x_i\|^2  \sum_{k=1}^d \| \N_k\|_F^2 \Big)^{1/2} \\
& \leq \frac{\gamma^2 B R}{\lambda} \,.
\end{align*}
Taking the sum over $k$ results in a bound of $(d) \leq \frac{\gamma^2 B
R^2 \sqrt{d T}}{\lambda}$ and completes the bound.

\section{Rademacher Analysis for Non-Relaxed Class}

In this section we analyze the generalization performance of the
original (non-relaxed) class of imputation-based hypotheses:
\begin{multline}
 \cG = \Big\{ h(\x_0, \z_0) \mapsto \w^\top \big(\x_0 + \Zb_0 \M^\top \x_0)
\\
  : \|\w\| \leq \Lambda, \|\M\|_F \leq \gamma \Big\}
\end{multline}

\begin{theorem}
If we assume a bounded regression problem $\forall y, ~|y| \leq B$
and $\forall \x,~ \|\x\| \leq R$, then the Rademacher complexity of the
hypothesis set $\cG$ is bounded as follows,
\begin{equation*}
  \Rad_T(\cG) \leq (1 + \gamma \sqrt{d}) \Lambda R \sqrt{T} \,.
\end{equation*}
\end{theorem}
\begin{proof}
We wish to bound
\begin{multline*}
  \E_\Ssigma \Big[ \sup_{\w, \M} \big| \sum_{i=1}^T \sigma_i
\w^\top(\x_i + \Zb_i \M^\top \x_i) \big|  \Big] \\
  \leq \E_\Ssigma \! \Big[ \sup_{\w} \big|\! \sum_{i=1}^T \! \sigma_i
\w^\top \! \x_i \big|  \Big] 
  \! + \E_\Ssigma \! \Big[ \sup_{\w, \M} \big| \! \sum_{i=1}^T \! \sigma_i
\w^\top \Zb_i \M^\top \! \x_i \big|  \Big] \,.
\end{multline*}
The first term is standard and is bounded by $\Lambda R \sqrt{T}$.
We bound the second term by first separating the terms depending on
$\w, \M$ and those depending on $\Ssigma$:
\begin{multline*}
  \sum_{i=1}^T \sigma_i \w^\top \! \Zb_i \M^\top \x_i
 =  \!\! \sum_{s,r=1}^d [\w]_r [\M]_{s,r} 
   \Big(\sum_{i=1}^T \sigma_i [\x_i]_s [\zb_i]_r \Big) \\
 \leq \underbrace{\Big( \sum_{s,r=1}^d [\w]_r^2 [\M]_{s,r}^2
\Big)^{\frac12}}_{\mathrm{(a)}}
  \underbrace{\Big(\sum_{s,r=1}^d \big(\sum_{i=1}^T \sigma_i [\x_i]_s [\zb_i]_r
\big)^2 \Big)^{\frac12}}_{\mathrm{(b)}} \,,
\end{multline*}
where the inequality follows from the Cauchy-Schwartz inequality. We
bound the first term using the fact that every term in the sum is
positive and adding an additional summation:
\begin{multline*}
\!\!\! \mathrm{(a)} \leq \Big( \!\big(\sum_{r=1}^d [\w]_r^2\big)\big(
\sum_{r,s=1}^d [\M]_{s,r}^2 \big) \!\Big)^{\frac12} 
  \!=\! \|\w\| \|\M\|_F \leq \Lambda \gamma .
\end{multline*}
The second term is bounded by first applying Jensen's inequality and
then the property that $\E_\Ssigma[ (\sum_{i} \sigma_i a_i)^2 ] =
\sum_{i} a_i^2$ for any constants $a_i$.
\begin{multline*}
\mathrm{(b)} \leq \Big( \sum_{s,r=1}^d \E_\Ssigma [(\sum_{i=1}^T
\sigma_i [\x_i]_s [\zb_i]_r)^2] \Big)^{\frac12} \\
  =  \Big(\sum_{i=1}^T \sum_{s,r=1}^d [\x_i]_s^2 [\zb_i]_r^2 \Big)^{\frac12}
  \leq \Big(\sum_{i=1}^T \| \x_i \|^2 \|\zb_i\|^2 \Big)^{\frac12} \\
  \leq R \sqrt{d T} \,.
\end{multline*}
Combining the two bounds and dividing by $T$ proves the theorem.
\end{proof}

\section{Experimental Setup}
We start by explaining corruption-dependent and corruption-independent
corruption processes. In the former case a 
probability of corruption is chosen uniformly at random from $[0,\beta]$ for
each feature (where $\beta$ can be tuned to induce more or less
corruption), independent of the other features and independent of data
instance (i.e. $\z_t$ is independent of $\x_t$).  In the latter case,
a random threshold $\tau_k$ is chosen for each feature uniformly
between $[0,1]$ as well as a sign $\sigma_k$ chosen uniformly from
$\{-1,1\}$. Then, if a feature satisfies $\sigma([\x_i]_k - \tau_k) >
0$, it is deleted with probability $\beta$, which again can be tuned
to induce more or less missing data. Table~\ref{table:data} shows the
average fraction of features remaining after being subject to each type
of corruption; that is, the total sum of features available over all
instances divided by the total number of features that would be
available in the corruption-free case.

The average error-rate along with one standard deviation is reported
over 5 trials each with a random fold of $1000$ training
points\footnote{With the exception of {\tt optdigits} in the online
setting where $3000$ points are used.}. In the batch setting the
remainder of the dataset in each trial is used as the test set.  When
applicable, each trial is also subjected to a different random
corruption pattern.  All scores are reported with respect to the best
performing parameters, $\lambda$, $C$, $\eta$ and $\gamma$, tuned across
the values $\{2^{-12}, 2^{-11}, \ldots, 2^{10}\}$.

\end{document}